\documentclass[submission,copyright,creativecommons]{eptcs}
\providecommand{\event}{ICLP 2025} 
\usepackage{times}
\usepackage{soul}
\usepackage{url}
\usepackage[utf8]{inputenc}
\usepackage{tikz}
\usetikzlibrary{positioning}
\usepackage{graphicx}
\usepackage{amsmath}
\usepackage{amssymb}
\usepackage{amsthm}
\usepackage{booktabs}
\usepackage{algorithm}
\usepackage{algorithmic}
\usepackage[switch]{lineno}

\newcommand{\knows}{\mathbf{K}}     
\newcommand{\only}{\mathbf{O}}      
\newcommand{\abd}{\mathbf{A}}       

\theoremstyle{plain}
\newtheorem{mydef}{Definition}[section]
\newtheorem{lemma}{Lemma}
\newtheorem{proposition}{Proposition}
\newtheorem{corollary}{Corollary}
\newtheorem{example}{Example}
\newtheorem{theorem}{Theorem}

\usepackage{iftex}

\ifpdf
  \usepackage{underscore}         
  \usepackage[T1]{fontenc}        
\else
  \usepackage{breakurl}           
\fi

\title{Propositional Abduction via Only-Knowing\\ A Non-Monotonic Approach}
\author{Sanderson Molick
\institute{Division of Humanities\\
Federal Institute of Pará\\
Pará, Brazil}
\email{smolicks@gmail.com}
\and
Vaishak Belle
\institute{School of Informatics\\
University of Edinburgh\thanks{Vaishak Belle was funded by a CISCO grant and a Royal Society University Research Fellowship.}\\
Edinburgh, UK}
\email{\quad vbelle@ed.ac.uk}
}
\def\titlerunning{Propositional Abduction via Only-Knowing}
\def\authorrunning{Molick, S. \& Belle, V.}
\begin{document}
\maketitle

\begin{abstract}

The paper introduces a basic logic of knowledge and abduction by extending Levesque's logic of only-knowing with an abduction modal operator defined via the combination of basic epistemic concepts. The upshot is an alternative approach to abduction that employs a modal vocabulary and explores the relation between abductive reasoning and epistemic states of only-knowing.
Furthermore, by incorporating a preferential relation into modal frames, we provide a non-monotonic extension of our basic framework capable of expressing different selection methods for abductive explanations. Core metatheoretic properties of non-monotonic consequence relations are explored within this setting and shown to provide a well-behaved foundation for abductive reasoning.

\end{abstract}

\section{Introduction}

Abductive reasoning is a central component for contemporary accounts of knowledge representation in artificial intelligence (see \cite{tsamoura2021neural} or \cite{sep-logic-ai}). An abductive problem in a logic $\mathcal{L}$ is often represented by a background theory $\Theta$ along with an observed event $\alpha$ such that $\Theta \not \models_{\mathcal{L}} \alpha$. The solution for the problem described by the pair $\langle \Theta, \alpha \rangle$ is a formula $\varphi$ (also called the \textit{explanans} or the \textit{explanation}) such that $\Theta \cup \{ \varphi \} \models_{\mathcal{L}} \alpha$. 
 
In this paper, we investigate the problem of selecting explanations $\varphi$ that are coherent with an agent's background knowledge $\Theta$, and that provide an explanatory account of an observed event $\alpha$. Accordingly, we introduce a novel extension of Levesque's \textbf{Logic of Only-Knowing} ($\mathcal{OL}$) capable of expressing abductive reasoning within the epistemic bounds of the agent's background knowledge, a modal framework called the \textbf{Logic of Only-Knowing and Abduction} (herafter denoted as $\mathcal{AOL}$). In  $\mathcal{AOL}$, abduction is treated as a derived modality defined through the interaction of basic epistemic concepts. The framework is constructed through the addition of an abduction operator $\abd$ defined in terms of the epistemic modalities $\only$ and $\knows$ (where $\only \varphi$ denotes that the agent ``only-knows $\varphi$'' and $\knows \varphi$ denotes that the agent ``knows $\varphi$''). As a result, abduction is understood as an epistemic process of knowledge acquisition, in accordance with Peircean accounts of abduction such as that explored by A. Aliseda (\cite{aliseda2000abduction}).


The problem of developing knowledge-based accounts of abduction via $\mathcal{OL}$ was first introduced by Levesque (\cite{levesque1989knowledge}) as a way to distinguish implicit from explicit beliefs. While other modal-based accounts of abduction, such as \cite{nepomuceno2017abductive}, address abduction within dynamic epistemic logics and related frameworks, our approach focuses on selecting explanations compatible with the agent's \emph{only-known} background knowledge, thus modeling abduction as a structured form of inference grounded in basic epistemic states. To this end, the paper introduces the basic logic of only-knowing and abduction ($\mathcal{AOL}$) along with a non-monotonic extension (the logic $\mathcal{AOL}^\prec$) by enriching the underlying modal framework with a preferential structure. Furthermore, the formal properties of the abductive modality in terms of consistency, explanatory power, and minimality, are explored within our framework.



\textbf{Overview}.
The syntax and semantics of our basic framework are presented in Section 2. Core properties of the abductive modality in capturing key aspects of logic-based accounts of abduction are also presented in Sections 2 and 3. Subsequently, in Section 4, we augment our modal frames via a plausibility semantics and introduce a non-monotonic consequence relation based on the logic of only-knowing and abduction. This move allows one to select minimal explanations compatible with a given abductive problem $\langle \Theta, \alpha \rangle$.
The section also introduces a set-theoretic approach to minimal explanations that allows a different criterion for selecting minimal explanations. In Section 5, we prove core metatheoretic results for our non-monotonic framework. Finally, in Section 6, the paper ends with a discussion of related work and the paths for future exploration.

\section{Preliminaries: the logic $\mathcal{AOL}$}

\subsection*{Syntax}

We consider a modal propositional language $\mathrm{L}_{\only,\abd}$ built from a countable set of atoms ($At$), the standard boolean connectives $\neg$, $\wedge$ and $\rightarrow$, and three epistemic operators: $\knows$, $\only$ and $\abd$. The set of formulas of $\mathrm{L}_{\only,\abd}$ (denoted by $Fm(\mathrm{L}_{\only,\abd})$) is inductively constructed as follows:


\begin{center}
$\varphi :== p \, \vert \, \neg \varphi \, \vert \, \varphi \wedge \varphi \, \vert \, \varphi \rightarrow \varphi \, \vert \, \knows \varphi \, \vert \, \only \varphi \, \vert \, \abd \varphi$, \\
\end{center}

\noindent where $p \in At$. We use $\psi$, $\varphi$, $\delta$, ... to denote arbitrary individual formulas, and $\Gamma$, $\Delta$, $\Theta$, ... to denote sets of formulas in $Fm(\mathrm{L}_{\only,\abd})$. We shall follow Levesque's presentation (\cite{levesque1990all}) and let $\only\varphi$ to be read as ``the agent \emph{only-knows} $\varphi$'' and $\knows\varphi$ to be read as ``the agent \emph{knows} or \emph{believes} $\varphi$''. In addition, $\abd\varphi$ will be read as ``the agent knows $\varphi$ \emph{by abduction}''. A formula will be called \textbf{abductive} if it is a boolean formula preceded only by the modal operator $\abd$.

\subsection*{Semantics}

We adopt Levesque's semantic convention, according to which an epistemic situation is represented by a pair $(\mathcal{W}, w)$, where $\mathcal{W}$ is a set of epistemic states and $w \in \mathcal{W}$ is designated as the \textbf{actual} state, with the usual truth-assignment to atomic formulas. The semantic structure is defined as follows:

\begin{mydef}
Where $\mathrm{L}_{\only,\abd}$ is a modal propositional language, a \textbf{Kripke model} for $\mathrm{L}_{\only,\abd}$ is a tuple $\mathcal{M} = \langle \mathcal{W}, \mathcal{R}, v \rangle$, where $\mathcal{W}$ is a set of epistemic states, $\mathcal{R}$ is a binary relation over $\mathcal{W}$ and $v$ is a mapping that assigns to each atom of $\mathrm{L}_{\only,\abd}$ a subset of $\mathcal{W}$. The definition of \emph{satisfaction} (or \emph{truth)} of a formula in a state $w \in \mathcal{W}$ of a model $\mathcal{M}$ is determined by the following clauses:
\begin{itemize}
    \item[(i)] $(\mathcal{M}, w) \models p$ iff $w \in v(p)$ and $p$ is an atom.
     \item[(ii)] $(\mathcal{M}, w) \models \neg \varphi$ iff $(\mathcal{M}, w) \not\models \varphi$. 
    \item[(iii)] $(\mathcal{M}, w) \models \varphi \wedge \psi$ iff $(\mathcal{M}, w) \models \varphi$ and $(\mathcal{M}, w) \models \psi$.
    \item[(iv)] $(\mathcal{M}, w) \models \varphi \rightarrow \psi$ iff  $(\mathcal{M}, w) \models \varphi$ implies $(\mathcal{M}, w) \models \psi$.
    \item[(v)] $(\mathcal{M}, w) \models \knows \varphi$ iff $(\mathcal{M}, w') \models \varphi$ for all $w' \in \mathcal{M}$ such that $wRw'$.
    \item[(vi)] $(\mathcal{M}, w) \models \only \varphi$ if for all states $w' \in \mathcal{W}$: $wRw'$ iff $(\mathcal{M}, w')\models\varphi$.
    \item[(vii)] $(\mathcal{M}, w) \models \abd \varphi$ iff $\exists \alpha: (\mathcal{M}, w) \models \only \alpha$ and $(\mathcal{M}, w) \models \knows ( \varphi \rightarrow \alpha)$. \\
\end{itemize}
\end{mydef}

According to clause $(vii)$, the abductive formula selects the explanation compatible with the agent's \textit{only-known} background knowledge. The use of the formula $\alpha$, combined with the modality $\only$, ensures that the abductive inference relies solely on the agent’s explicitly represented background knowledge, requiring no additional epistemic closure or stronger assumptions. The following example illustrates the rationale for the behavior of the modality $\abd$:

\begin{example}
\label{ex:di1}
Let $\Theta$ denote a set of diagnostic principles represented as only-knowing implications, $F$ denote the set of symptoms of a patient, and $E$ be the set of possible diagnostic explanations. Consider the following medical situation 
\begin{center}
$\Theta = \{ \only (cold \rightarrow cough), $ \\ $\only(flu \rightarrow (cough \wedge fever))$, \\ $ \only (pneumonia \rightarrow (chest\_ pain \wedge cough \wedge fever)) \}$\\
$F = \{ fever, cough   \}$ \\
$E = \{ flu \}$.\\
In this example, the best abductive explanation for the patient's symptom is the condition $flu$. Let $\mathcal{M}$ be a model for the pair $\langle \Theta, \alpha \rangle$ and let $\alpha := cough \wedge fever$. Since both $\only \alpha$ and $\knows (\varphi \rightarrow \alpha)$ hold, according to the semantic clause of the abduction operator $\abd$, the diagnostician can conclude the explanation $\abd \varphi$, where $\varphi = flu$.
In this framework, the doctor does not directly know that the patient has the condition $flu$\footnote{Note that even though $cold$ could be abductively inferred, the condition $flu$ is the most compatible with the set of symptoms.}. Instead, the diagnosis arises as an \textbf{abductive inference}, derived from known symptoms and constrained hypotheses, given that the symptoms $cough \wedge fever$ are coherent with the hypothesis $flu$.
The combination of $\knows$ and $\only$ captures two distinct roles: 
$\only$ restricts the agent’s knowledge base to certain hypotheses, while 
$\knows$ provides the deductive link from hypotheses to observed symptoms (or events). The abductive conclusion $\abd \varphi$ is a product of the agent's \emph{only-known} knowledge base and what they can infer from it.
 
\end{center}
\end{example}

 \noindent 
 
 The treatment of abduction as affirming the antecedent is a strategy pursued by different authors (such as \cite{meheus2002ampliative}, \cite{aliseda2000abduction} or \cite{flach2000abduction}) and central for modelling Peirce's notion of abduction. 
 The kind of abduction illustrated by situations like Example \ref{ex:di1} is also referred to by G. Schurz (\cite{schurz2008patterns}) as \textbf{selective abduction}, as it depends on selecting the best possible explanation from a predefined set of known hypotheses. According to the semantic clause of the 
 $\abd$ operator, the only-known particle selects the explanation that is part of the implicit knowledge of the agent. 

As illustrated by the example, selecting explanations in accordance with the only-known background may prove useful in medical scenarios, where the explanation for a condition have to be selected without resorting to external sources of information.
The use of formal methods for modelling abduction in medical scenarios is called \textbf{abductive diagnosis} by R. Brachman (\cite{brachman2004knowledge}). The employment of the abductive modality $\abd$ allows one to check for the best explanation compatible with the initial set of conditions and diagnostic principles. As a consequence, abductive diagnosis is treated as a form of selective abduction within our framework. We introduce additional definitions in the following:

\begin{mydef}{(Local consequence)}
We shall write $\Gamma \models \varphi$ to denote that for every Kripke model $\mathcal{M} = \langle \mathcal{W}, \mathcal{R}, v \rangle$ and every world $w \in \mathcal{W}$, if $(\mathcal{M}, w) \models \gamma$ for all $\gamma \in \Gamma$, then $(\mathcal{M}, w) \models \varphi$.


\end{mydef}

\begin{mydef}{(Global validity)}
A formula $\varphi$ will be called \textbf{globally valid} with respect to a class of models $\mathcal{F}$ (written $\models^{\mathcal{F}} \varphi$) if, for every model $\mathcal{M} = \langle \mathcal{W}, \mathcal{R}, v \rangle$ in $\mathcal{F}$, and for every world $w \in \mathcal{W}$, $(\mathcal{M}, w) \models \varphi$ holds. 

\end{mydef}

\begin{mydef}{(Abductive explanation)}
A formula $\varphi$ will be called an \textbf{explanation} for the abduction problem $\langle \Theta, \alpha \rangle$ if $\varphi$ is an abductive formula and $\Theta \cup \{ \varphi \} \models \alpha$.  
\end{mydef}




    The following shows how to guarantee the existence of a formula $\alpha$ in case $(\mathcal{M}, w) \models \abd \varphi$ holds. In what follows, given a kripke model $\mathcal{M} = \langle \mathcal{W}, \mathcal{R}, v \rangle$, let $R(w) = \{ w'\in \mathcal{W}: wRw'   \}$ and $R_\varphi(w) = \{ w'\in R(w): (\mathcal{M}, w') \models \varphi \}$.

\begin{theorem}
Where $\Theta$ is a finite non-empty set of only-known formulas true in a model $\mathcal{M}$. The following statement holds for every $\phi \in Fm(\mathrm{L}_{\only,\abd})$: 
\begin{itemize}
    \item If $(\mathcal{M}, w) \models \abd \varphi$ holds, there is a formula $\alpha$ such that both $(\mathcal{M}, w) \models \only \alpha$ and $(\mathcal{M}, w) \models \knows (\varphi \rightarrow \alpha)$ hold. 
    
\end{itemize}

\end{theorem}

\begin{proof}
Suppose $(\mathcal{M}, w) \models \abd \varphi$ holds in a model $\mathcal{M}$. The proof consists in showing how $\alpha$ can be constructed for $\Theta \neq \emptyset$. 
Let $\Theta = \{ \psi_1, \psi_2, ..., \psi_n \}$ be the set of true only-known formulas in a state $w$. By the semantic clause for only-knowing, we obtain that $(\mathcal{M}, w) \models \only \varphi$ iff $R_\varphi(w) \neq \emptyset$. Now let $\alpha := \bigwedge \{ \varphi : (\mathcal{M}, w') \models \varphi$ and $w' \in R(w) \}$. Note that $\alpha = (\bigwedge \psi_i)$, for $1 \leq i \leq n$. From the construction, we know that $(\mathcal{M}, w') \models \varphi \Rightarrow (\mathcal{M}, w') \models \alpha$ holds in every $w'$. Hence $\varphi \rightarrow \alpha$ is true in all accessible states from $w$. As a result, by the semantic clauses for $\knows$ and $\only \alpha$, $(\mathcal{M},w) \models \knows (\varphi \rightarrow \alpha)$ and $(\mathcal{M},w) \models \only \alpha$ follow.


\end{proof}

The abduction operator is sensitive to the addition of new information in a model, a central feature for modeling abductive reasoning.
Consider a model $\mathcal{M}_1$ with $\mathcal{W} = \{ w_1 \}$ and the following validities: $(\mathcal{M}, w_1) \models fever$, $(\mathcal{M}, w_1) \models flu$ and $(\mathcal{M}, w_1) \not\models cold$. Now let $\alpha := flu$ and $\varphi := fever$. Note that $(\mathcal{M}, w_1) \models \only flu$ and $(\mathcal{M}, w_1) \models fever \rightarrow flu$. Hence $(\mathcal{M}, w_1) \models \knows (fever \rightarrow flu)$ and $(\mathcal{M}, w_1) \models \abd fever $. Now consider $\mathcal{W} = \{ w_1, w_2 \}$ such that $w_1Rw_2$ and the following valuation: $(\mathcal{M}, w_2) \models fever$, $(\mathcal{M}, w_2) \not\models flu$ and $(\mathcal{M}, w_2) \models cold$. The conclusion $\abd fever$ is blocked. 
The following are central properties for logic-based accounts of abduction (as proposed by Aliseda (\cite{aliseda2000abduction})) and shown to be in accordance with the $\abd$ modality:


\begin{mydef}
Let $\langle \Theta, \alpha \rangle$ be an abductive problem. The following states central properties for any $\langle \Theta, \alpha \rangle$: \\
(Consistency) $\Theta \cup \varphi$ is consistent and $\Theta, \varphi \models \alpha$; \\
(Explainability) (i) $\Theta \not\models \alpha$; 
(ii) $\varphi \not\models \alpha$; and (iii) $\Theta, \varphi \models \alpha$. 

\end{mydef}

\begin{theorem}
Let  $\langle \Theta, \alpha \rangle$ be an abductive problem, where $\Theta$ is a finite set of only-known implications. The following statements are true for any $\Theta \cup \alpha \subseteq Fm(\mathrm{L}_{\only,\abd})$: \\
(i)  If $(\mathcal{M}, w) \models \abd \varphi$ holds, then Explainability is true for $\langle \Theta, \alpha \rangle$\\    
(ii) If $(\mathcal{M}, w) \models \abd \varphi$ holds, then Consistency is true for $\langle \Theta, \alpha \rangle$.

\end{theorem}

\begin{proof}
Suppose $(\mathcal{M}, w) \models \abd \varphi$ holds in a model $\mathcal{M}$. Theorem 1 and the semantic clause for $\abd$ are sufficient for meeting the conditions for explainability. \\
For Consistency, assume for a contradiction that $\Theta \cup \varphi$ is inconsistent. Hence, there is a formula $\neg \varphi \in \Theta$ such that $(\mathcal{M}, w) \models \neg \varphi$. Given that $\Theta$ is a set of
only-known implications, we know that $\varphi$ has the form $\only(\alpha \rightarrow \beta)$, for some $\alpha, \beta \in Fm(\mathrm{L}_{\only,\abd})$. Therefore $(\mathcal{M}, w) \models \neg \only (\alpha \rightarrow \beta)$, i.e., there is a state $w'$ such that $wRw'$ and $(\mathcal{M}, w') \not\models  (\alpha \rightarrow \beta)$. Now, by the semantic clause for $\abd$, we know that $(\mathcal{M}, w) \models \knows (\alpha \rightarrow \beta)$. As a consequence, $(\mathcal{M}, w') \models (\alpha \rightarrow \beta)$ holds for every $w'$ such that $wRw'$, and a contradiction follows.
\end{proof}





\section{Core properties of abduction}

In this section, we introduce some core properties of $\mathcal{AOL}$. We start by noting that the modalities $\knows$ and $\abd$ do not collapse. While $\models \only \varphi \rightarrow \abd \varphi$ holds, we have $\not\models \abd \varphi \rightarrow \only \varphi$. We start by defining the set of models capable of validating instances of abductive inference.

\begin{proposition}

For every $\varphi, \psi \in Fm(\mathrm{L}_{\only,\abd})$: $\only \varphi, \only (\psi \rightarrow \varphi) \models \abd \psi$ holds in all reflexive models.    
\end{proposition}

\begin{proof}
Let $\mathcal{M}$ be a reflexive model such that $\mathcal{W} = \{ w_1, w_2 \}$ and $\mathcal{R} = \{ w_1 R w_2, w_1 R w_1, w_2 R w_2  \}$. Now assume that $(\mathcal{M}, w) \models \only (\psi \rightarrow \varphi)$ and $(\mathcal{M}, w) \models \only \varphi$ hold. Note that since $w$ is reflexive, it follows that $(\mathcal{M}, w) \models (\psi \rightarrow \varphi)$ and $(\mathcal{M}, w) \models \knows (\psi \rightarrow \varphi)$. Let $\alpha := \varphi$. By the semantic clause for $\abd$, we obtain $(\mathcal{M}, w) \models \abd \psi$.

\end{proof}

\begin{proposition}
The following inferences are valid in the logic of only-knowing and abduction $\mathcal{AOL}$: \\

\noindent (i). $\abd\varphi, \abd\psi \models \abd(\varphi \wedge \psi)$ \\
(ii). $\abd\varphi, \abd\psi \models \abd(\varphi \vee \psi)$ \\
(iii). $\abd \varphi,  \abd(\varphi \rightarrow \delta) \models \abd\delta$ 
\end{proposition}
    
\begin{proof}
Suppose  $(\mathcal{M}, w) \models \abd\varphi$ and $(\mathcal{M}, w) \models \abd\psi$ hold in a given model $\mathcal{M}$. By the semantic clause for $\abd$, there is an $\alpha$ such that  $(\mathcal{M}, w) \models \only \alpha$ and $(\mathcal{M}, w) \models \knows(\varphi \rightarrow \alpha)$. By the same reasoning, there is a $\beta$ such that $(\mathcal{M}, w) \models \only \beta$ and $(\mathcal{M}, w) \models \knows(\psi \rightarrow \beta)$. Now let $\alpha:= \varphi$ and $\beta := \psi$. By the semantic clause for $\only$, $(\mathcal{M}, w) \models \varphi \wedge \psi$. Now set $\gamma = (\varphi \wedge \psi)$. Hence $(\mathcal{M}, w) \models \only \gamma$ and $(\mathcal{M}, w) \models \knows (\gamma \rightarrow \gamma)$. Finally, the semantic clause for $\abd$ allows us to conclude $(\mathcal{M}, w) \models \abd(\varphi \wedge \psi)$. The proof of (ii) and (iii) follows by analogous reasoning.
\end{proof}

It is worth highlighting that the $\mathcal{OL}$-fragment is unable to account for interesting instances of abductive inferences based on only-known formulas in reflexive models. Based on the model $\mathcal{M}$ from Proposition 1, it is easy to see that $(\mathcal{M}, w) \models \only (\psi \rightarrow \varphi)$ and $(\mathcal{M}, w) \models \only \varphi$, but also $(\mathcal{M}, w) \not\models \only \psi$. 
Although some models without reflexivity may also validate Proposition 1, restricting $\mathcal{AOL}$ to reflexive model can be helpful  to avoid problematic cases such as the following: let  $\mathcal{M}$ be a model such that $\mathcal{W} = \{ w_1, w_2 \}$, $\mathcal{R} = \{ w_1 R w_2,  w_2 R w_1  \}$ and the following valuation:
$(\mathcal{M}, w_1) \models \alpha$, $(\mathcal{M}, w_1) \models \varphi \rightarrow \alpha$, $(\mathcal{M}, w_2) \not\models \alpha$ and 
$(\mathcal{M}, w_2) \not\models \varphi \rightarrow \alpha$. By the semantic clause for $\abd$, we obtain $(\mathcal{M}, w_2) \models \knows (\varphi \rightarrow \alpha)$ and $(\mathcal{M}, w_2) \models \abd \varphi$. However, one may also have $(\mathcal{M}, w_2) \not\models \varphi$. In models like this, the agent is allowed to infer by abduction a formula that is false at some state. By augmenting the model via reflexivity, one cannot obtain both $(\mathcal{M}, w_2) \models \knows (\varphi \rightarrow \alpha)$ and $(\mathcal{M}, w_2) \not\models  \varphi \rightarrow \alpha$. Unfortunately, even \textbf{S5}-models of $\mathcal{AOL}$ are not sufficient for ensuring non-vacuity of abduction as follows:

\begin{proposition}[Non-Vacuity of Abduction]
For every $\varphi \in Fm(\mathrm{L}_{\only,\abd})$: If $(\mathcal{M}, w) \models \abd \varphi$, then $(\mathcal{M}, w) \not\models \knows \neg \varphi$.

\end{proposition}

One should note, however, that asserting $(\mathcal{M}, w) \models \abd \varphi$ is not equivalent to $(\mathcal{M}, w) \models \varphi$. As a result, $(\mathcal{M}, w) \models \abd \varphi$ and $(\mathcal{M}, w) \models \neg \varphi$ are not logically contradictory. Even in models where the conditional $\varphi \rightarrow \alpha$ holds vacuously true, the agent may still wish to consider what can be coherently inferred based on the basis of her only-known background knowledge alone. Situations like this may also arise in medical diagnosis, where the available information may allow the diagnostician to entertain mutually incompatible explanatory hypotheses. In the absence of definitive evidence, the diagnostician must reason abductively based on her only-known background knowledge, while remaining aware that some of the entertained hypotheses are incompatible with some evidence or may ultimately prove to be false. In Section 4, we avoid models with vacuously true conditional by augmenting our modal framework with a preferential semantics. The following result shows that models where non-vacuity of abduction fails have an equivalent submodel where non-vacuity holds.

\begin{lemma}
\label{l1}
Let $\mathcal{M} = \langle \mathcal{W}, \mathcal{R}, v \rangle$ be a Kripke model and $w \in \mathcal{W}$ a state such that $(\mathcal{M}, w) \models \abd \varphi$. Then there exists a submodel $\mathcal{M}' = \langle \mathcal{W}', \mathcal{R}', w' \rangle$, and $w' \in \mathcal{W}'$ for which the following holds: 
\begin{center}
$(\mathcal{M}', w') \models \abd \varphi$ and $(\mathcal{M}', w') \not\models \knows \neg \varphi$ hold iff
$\mathcal{W}' = \{ w' \in W \, \vert \, (\mathcal{M}, w') \models \Theta$ and $(\mathcal{M}, w') \not\models \neg \varphi \}$ is non-empty, \\ where $\Theta$ denotes the agent's knowledge base.
\end{center}
\end{lemma}

\begin{proof}
Let $\mathcal{M} = \langle \mathcal{W}, \mathcal{R}, v \rangle$ be a model with a state $w$ such that $(\mathcal{M}', w') \models \abd \varphi$ and $(\mathcal{M}', w') \not\models \knows \neg \varphi$ hold. Define the submodel model $\mathcal{M}' = \langle \mathcal{W}', \mathcal{R}', v' \rangle$ as follows: 
\begin{itemize}
\item  $\mathcal{W}' = \{ w' \in W \, \vert \, (\mathcal{M}, w') \models \Theta$ and $(\mathcal{M}, w') \not\models \neg \varphi \}$,
\item $\mathcal{R}' = \mathcal{R} \cap (\mathcal{W}' \times \mathcal{W}')$, 
\item $v' = v(p) \cap \mathcal{W}'$,
\end{itemize}
By construction, it is easy to see that $\mathcal{W}'$ is non-empty. From l.h.s. to r.h.s., from the fact that  $(\mathcal{M}, w') \models \Theta$ holds and the assumption, $(\mathcal{M}', w') \models \abd \varphi$ follows.
\end{proof}

\begin{lemma}
\label{l2}
For any $\varphi \in Fm(\mathrm{L}_{\only,\abd})$, it holds that:
\begin{center}
If $(\mathcal{M}, w) \models \varphi$, then $(\mathcal{M}', w') \models \varphi$,   
\end{center}    
where $\mathcal{M}'$ is the submodel of $\mathcal{M}$.
\end{lemma}

\begin{proof}
By induction on the length of $\varphi$. For \( \varphi = p \), where \( p \) is an atom. By construction, \( v'(p) = v(p) \cap \mathcal{W}' \). Thus, \( w \in v(p) \) implies \( w' \in v'(p) \). Therefore \( (\mathcal{M}', w') \models p \). For \( \varphi = \neg \psi \). By induction hypothesis, \( (\mathcal{M}, w) \models \phi \) implies \( (\mathcal{M}', w') \models \phi \). Hence, \( (\mathcal{M}, w) \not\models \phi \) implies \( (\mathcal{M}', w') \not\models \phi \). Given the clause for negation, \( (\mathcal{M}', w) \models \neg \phi \). The rest of the proof follows as usual. The case of modal operators follows by the restriction of the accessibility relation.
    

\end{proof}

\begin{theorem}
Where $\mathcal{M} = \langle \mathcal{W}, \mathcal{R}, v \rangle$ is a Kripke model in which non-vacuity of abduction fails at a state $w \in \mathcal{W}$. There exists a submodel $\mathcal{M}' = \langle \mathcal{W}', \mathcal{R}', v' \rangle$ and $w' \in \mathcal{W}'$ such that the following holds: \\
(i) $(\mathcal{M}', w') \models \Theta$, for any $\Theta \subseteq Fm(\mathrm{L}_{\only,\abd})$ such that $(\mathcal{M}, w) \models \Theta$; and \\
(i) $(\mathcal{M}', w') \models \abd \varphi$ implies $(\mathcal{M}', w') \not\models \knows \neg \varphi$.

\end{theorem}

\begin{proof}
Assume $\mathcal{M} = \langle \mathcal{W}, \mathcal{R}, v \rangle$ is a model in which non-vacuity of abduction fails at a state $w \in \mathcal{W}$. By Lemma \ref{l1}, there exists a submodel $\mathcal{M}' = \langle \mathcal{W}', \mathcal{R}', v' \rangle$. From Lemma \ref{l2}, there is a state $w' \in \mathcal{W}'$ such that $(\mathcal{M}, w) \models \varphi$ implies $(\mathcal{M}', w') \models \varphi$. Moreover, $(\mathcal{M}', w') \models \Theta$.
Now assume $(\mathcal{M}', w') \models \abd \varphi$. By construction of $\mathcal{M}'$, $(\mathcal{M}', w') \not\models \knows \neg \varphi$ follows.
\end{proof}

\section{Abductive problems and minimal explanations}

The task of selecting explanations that are compatible with an agent's background knowledge often requires comparing them according to specific criteria.\footnote{See, for instance, \cite{meheus2007adaptive} and \cite{delrieux2004abductive}, who argue that selective abduction involves some form of defeasible reasoning.} In particular, the ability to evaluate and compare alternative sets of explanations proves especially valuable in scenarios where multiple explanations are available for the same event (Aliseda in \cite{aliseda2006abductive} points out how selection of minimal explanations is a central property of logic-based approaches to abduction). For example, as illustrated by Example \ref{ex:di1}, if the doctor discovers that the cough is not a revelant symptom, she must reconsider which explanation to adopt. Although the logic $\mathcal{AOL}$ is sensitive to new information at the level of individual models, the logic remains monotonic at the level of the consequence relation. 

In this section, we augment our Kripke models with a plausibility relation $\prec$. The resulting logic is a non-monotonic extension of $\mathcal{AOL}$ (hereafter denoted as $\mathcal{AOL}^{\prec}$). Modeling abductive reasoning via non-monotonic consequence relations can be found in \cite{meheus2002ampliative} or \cite{meheus2007adaptive} via adaptive logics. While these authors treat abduction through non-monotonic ampliative approaches to validate inferences of the form $(\varphi \wedge \psi), \psi \models \varphi$, the $\mathcal{AOL}^{\prec}$ framework focuses on inferences compatible with the only-knowing background. We start by introducing basic definitions.

\begin{mydef}
Where $\mathrm{L}_{\only,\abd}$ is a modal language, a \textbf{plausibility Kripke model} for $\mathrm{L}_{\only,\abd}$ is a tuple $\mathcal{M} = \langle \mathcal{W}, \mathcal{R}, \prec , v \rangle$, where $\mathcal{W}$ is a set of epistemically possible states,  $\mathcal{R}$ is a binary relation over $\mathcal{W}$, $ \prec $ is a plausibility ordering over $\mathcal{W}$, where $w \prec w'$ means that $w$ is \textbf{preferred (more plausible)} than $w'$, and $v$ is a mapping that assigns to each atom of $\mathrm{L}_{\only,\abd}$ a subset of $\mathcal{W}$. The definition of \emph{satisfaction} (or \emph{truth)} of a formula in a state $w \in \mathcal{W}$ of a model $\mathcal{M}$ remains similar to that of Definition 2.1. The altered clauses are the following:
\begin{itemize}
 \item[(iv)] $(\mathcal{M}, w) \models \varphi \mathbin{>} \psi$ iff  $(\mathcal{M}, w) \models \psi$ holds in all $\prec$-minimal states accessible from $w$ such that $(\mathcal{M}, w) \models \varphi$.
\item[(vii)] $(\mathcal{M}, w) \models \abd \varphi$ iff $\exists \alpha: (\mathcal{M}, w) \models \only \alpha$ and $(\mathcal{M}, w) \models ( \varphi \mathbin{>} \alpha)$ in all $\prec$-minimal states accessible from $w$. 
\end{itemize}
\end{mydef}

Hereafter we shall assume that $\prec$ is a transitive and connected relation (i.e., $w \prec w'$ or $w' \prec w$ holds for all $w, w' \in \mathcal{W}$).

\begin{mydef}
 A state $w$ is called \textbf{$\prec$-minimal} (or just \textbf{minimal}) if there is no $w' \prec w$ and it is called \textbf{$\varphi$-minimal} in case $(\mathcal{M}, w) \models \varphi$ and there is no $w' \in \mathcal{W}$ such that $w' \prec w$ and $(\mathcal{M}, w') \models \varphi$.  
\end{mydef}


\begin{mydef}
A formula $\varphi$ will be called a \textbf{preferential consequence} of a set of formulas $\Gamma$ (denoted $\Gamma \models_{\prec} \varphi$) iff
$(\mathcal{M}, w) \models \varphi$ holds in all minimal states $w$ such that 
$(\mathcal{M}, w) \models \gamma$ for every $\gamma \in \Gamma$.
\end{mydef}



\begin{mydef}{(Minimal explanation)}
A formula $\varphi$ will be called a \textbf{minimal explanation} for the abduction problem $\langle \Theta, \alpha \rangle$ if $\varphi$ is an abductive formula and $\Theta \cup \{ \varphi \} \models_{\prec} \alpha$.  
\end{mydef}



 It is easy to see that Consistency and Explainability hold for abductive problems in the preferential consequence $\models_\prec$. Moreover, the semantic clause for the preferential conditional guarantees non-vacuity of abduction. The following example illustrates a situation where the $\mathcal{AOL}$ consequence relation fails to select a single explanation, while the preferential entailment can do it:

\begin{example}
\label{ex:di2}
\textit{Let $\langle \Theta, \beta \rangle$ be an abduction problem where $\Theta = \{ ((\alpha \mathbin{>} (\gamma \vee \delta)) \vee \delta) \vee (\psi \mathbin{>} ((\gamma \vee \delta) \vee \gamma)) \}$ and $\beta = (\alpha \vee \psi)$. Consider a plausibility model $\mathcal{M}$ such that $\mathcal{W} = \{ w_1 , w_2\}$, $\mathcal{R} = \{ w_1 R w_1, w_2Rw_2 \}$ with $w_1 \prec w_2$ and the following scenario: \\}

\begin{center}
\begin{tabular}{c|c|c|c|c|c|c}
$\mathcal{W}$ & $(\alpha \mathbin{>} (\gamma \vee \delta)$ & $(\psi \mathbin{>} (\gamma \vee \delta)$ & $\gamma$ & $\delta$ & $\abd \alpha$ & $\abd \psi$ \\
\hline
\hline
$w_1$ & $\surd$  & $\times$ & $\surd$ & $\times$ & $\surd$ & $\times$ \\
\hline
$w_2$ & $\times$  & $\surd$ & $\times$ & $\surd$ & $\times$ & $\surd$ \\
\hline
\end{tabular}
\end{center}

According to the above example, we have $(\mathcal{M}, w_1) \models \gamma$ and therefore $(\mathcal{M}, w_1) \models \gamma \vee \delta$. Moreover, by the semantic clause for implication, $(\mathcal{M}, w_1) \models \alpha \mathbin{>} (\gamma \vee \delta)$. By the semantic clause $\only$, $(\mathcal{M}, w_1) \models \only (\gamma \vee \delta)$ follows. Finally, the semantic clause for $\abd$ gives us $(\mathcal{M}, w_1) \models \abd \alpha$. By following the same rationale for the state $w_2$, we obtain $(\mathcal{M}, w_2) \models \abd \psi$. Now, the $\mathcal{AOL}$-consequence relation is unable to select one explanation. However, since $w_1$ is a minimal state, the preferential consequence will select $\abd \alpha$ as the minimal explanation.

\end{example}





\subsection{Non-monotonic properties}

The properties of the order $\prec$ guarantees that the preferential consequence constructed over the logic $\mathcal{AOL}$ enjoy core metaproperties for non-monotonic logics (see \cite{sep-logic-nonmonotonic} or \cite{makinson2003bridges}) and thus consists in a well-behaved consequence relation\footnote{Here a non-monotonic logic $\mathcal{L}_{nm}$ is defined as a structure $\langle \models_m, \prec \rangle$, where $\models_m$ is a non-trivial monotonic consequence relation and $\prec$ is a preferential semantics.}. The proofs of each of the following theorems are standard and can be found in \cite{makinson2003bridges}.


\begin{theorem}[Supraclassicality]
For any $\Gamma \cup \{  \varphi \} \subseteq Fm(\mathrm{L}_{\only,\abd})$: If $\Gamma \models \varphi$ then $\Gamma \models_{\prec} \varphi$.
\end{theorem}


\begin{theorem}[Reflexivity]
For any $\Gamma \cup \{  \varphi \} \subseteq Fm(\mathrm{L}_{\only,\abd})$: If $\varphi \in \Gamma$ then $\Gamma \models_{\prec} \varphi$.
\end{theorem}


\begin{theorem}[Cautious Monotony]
For any $\Gamma \cup \{  \varphi, \psi \} \subseteq Fm(\mathrm{L}_{\only,\abd})$: If $\Gamma \models_{\prec} \psi$ and $\Gamma \models_{\prec} \varphi$, then $\Gamma \cup \{ \varphi \}  \models_{\prec} \psi$.
\end{theorem}


\begin{theorem}[Cautious Transitivity]
For any $\Gamma \cup \{  \varphi, \psi \} \subseteq Fm(\mathrm{L}_{\only,\abd})$: If $\Gamma \models_{\prec} \varphi$ and $\Gamma \cup \{ \varphi \} \models_{\prec} \psi$, then $\Gamma \models_{\prec} \psi$
\end{theorem}






\subsection{Additional ways of selecting explanations}

In this section, we introduce additional restrictions for comparing different sets of minimal explanations relative to a given abduction problem. Example \ref{ex:di2} illustrated a situation in which more than one explanation was available relative to the abduction problem. In these situations, the $\mathcal{AOL}$-consequence is unable to select a single explanation. Due to the fact that $w_1 \prec w_2$, the preferential consequence allows us to select a unique explanation. In scenarios where multiple explanations are available, the preferential consequence allows not only the selection of minimal explanations, but also the ability to impose additional constraints on the set of minimal explanations\footnote{Different ways of comparing explanations via logic-based approaches to abduction are found in \cite{eiter1995complexity} via a probabilistic semantics or in \cite{mayer1995propositional} via a possible-world semantics.}. In the following, we show that different standards for the selection of minimal explanations can be expressed via the logic $\mathcal{AOL}^\prec$.

\begin{mydef}
A set of formulas $\Delta$ will be called a \textbf{set of minimal explanations} for the abduction problem $\langle \Theta, \alpha \rangle$ if every $\delta \in \Delta$ is a minimal explanation (also written $\Delta_{min}$).
\end{mydef}

In what follows, we introduce different ways of comparing minimal explanations by introducing further restrictions on the set of explanations selected by the preferential consequence $\models_\prec$. Let $\Pi(\Theta, \alpha) = \{ \varphi \, \vert \, \Gamma, \varphi \models_\prec \alpha \}$, i.e., $\Pi(\Theta, \alpha)$ denotes the set of all minimal explanations of the problem $\langle \Theta, \alpha \rangle$ in a model $\mathcal{M}$.

\begin{mydef}{(Subset-minimality selection)}
Let $\Pi(\Theta, \alpha)^\subset = \{ \Delta \in \Pi(\Theta, \alpha) \, \vert \, \forall \Delta' \subset \Delta, \Delta' \not\in  \Pi(\Theta, \alpha) \}$. The set $\Delta$ is called the \textbf{subset-minimal} set of explanations.
\end{mydef}




\begin{mydef}{(Cardinality selection)}
Let $\Pi(\Theta, \alpha)^< = \{ \Delta \in \Pi(\Theta, \alpha) \, \vert \, \forall \Delta'$ such that $\vert \Delta' \vert < \vert \Delta\vert , \Delta' \not\in  \Pi(\Theta, \alpha) \}$. The set 
$\Delta$ is called the \textbf{cardinality-minimal} set of explanations.
\end{mydef}

For the priorization comparison we enhance the sets of minimal explanations via a priorization function $g: I \rightarrow \Delta$ such that $\delta_i \leq \delta_k$ or $\delta_k \leq \delta_i$ for each $i, k \in \Delta$ (and $I$ is a countable index set). We will say that the index $i$ is the \textbf{priority level} of the explanation $\delta_i$. Given two sets of minimal explanations $\Delta_{1}$ and $\Delta_{2}$, we shall write $\Delta_{1} \sqsubseteq \Delta_{2}$ in case $i \leq k$ for some $\delta_{i} \in \Delta_{1}$ and $\delta_{k} \in \Delta_{2}$.

\begin{mydef}{(Priorization selection)}
Let $\Pi(\Theta, \alpha)^{\sqsubseteq} = \{ \Delta \in \Pi(\Theta, \alpha) \, \vert \,  \forall \Delta' \sqsubseteq \Delta , \Delta' \not\in  \Pi(\Theta, \alpha) \}$. The set  $\Delta$ is called the \textbf{priorization} set of explanations.
\end{mydef}


    

    The following example illustrates the differences between each standard of selection.

\begin{example}
\textit{Consider the following medical situation: 
\begin{center}
$\Theta = \{ \only (common\_ cold \rightarrow (sore\_ throat \vee cough)), $ \\ $\only (strep\_ throat \rightarrow (sore\_ throat \wedge fever))$, \\ 
$\only (allergies \rightarrow (headache \vee itchy\_ eyes)) \}$\\
$F = \{ fever, sore\_ throat, headache  \}$ \\
$E = \{ common\_ cold, sthrep\_ throat, allergies \}$. \\
\end{center}
Let $\Pi (\Theta, F) = \{ \Delta_{1}, \Delta_{2} \}$, where $\Delta_{1} = \{ allergies, sthrep\_ throat \}$ and $\Delta_{2} = \{ common\_ cold, sthrep\_ throat \}$. Both the cardinality and the subset selections are unable to pick a single set of minimal explanations. One way of handling this kind of situation is via the priorization selection. For this, assign the following priority levels: $allergies_1$, $sthrep\_ throat_2$ and  $common\_ cold_3$. According to the priorization selection, the set $\Delta_{1}$ will be selected as the most plausible for the abduction problem. The priorization approach may prove useful in medical scenarios where the treatment of a certain symptom needs to receive priority over others.}
\end{example}




By further restricting the set of explanations selected by the preferential consequence, one obtains the following additional consequence relations: $\models_{\prec}^{s}$, $\models_{\prec}^{c}$ and $\models_{\prec}^{p}$, where each consequence relation may be introduced as follows:

\begin{mydef}{(s-consequence)}
$\Gamma \models_{\prec}^{s} \varphi$ iff $\Delta \subseteq \Gamma$, where $\Delta \in \Pi(\Theta, \varphi)^\subset$ and $\Delta \models_{\prec} \varphi$.
\end{mydef}

\begin{mydef}{(c-consequence)}
$\Gamma \models_{\prec}^{c} \varphi$ iff $\Delta \subseteq \Gamma$, where $\Delta \in \Pi(\Theta, \varphi)^<$ and $\Delta \models_{\prec} \varphi$.
\end{mydef}

\begin{mydef}{(p-consequence)}
$\Gamma \models_{\prec}^{p} \varphi$ iff $\Delta \subseteq \Gamma$, where $\Delta \in \Pi(\Theta, \varphi)^{\sqsubseteq}$ and $\Delta \models_{\prec} \varphi$.
\end{mydef}

In the next section, we show under which conditions each $\models_{\prec}^{\ast}$, where $\ast \in \{ s, c, p \}$, coincides with $\models_{\prec}$. Furthermore, Theorem 7 in the next section guarantees that $\Pi(\Theta, \alpha) \neq \emptyset$ for any abduction problem $\langle \Theta, \alpha \rangle$. The table below shows the properties for each notion
of consequence relation.\\





\begin{table}[h!]
\centering
\begin{tabular}{lcccccc}
\toprule
\textbf{Property} & $\models$ & $\models_{\prec}$ 
& $\models_{\prec}^{s}$ & $\models_{\prec}^{c}$ & $\models_{\prec}^{p}$ \\
\midrule
Supraclassicality & $\checkmark$ & $\checkmark$ & $\checkmark$ & $\checkmark$ & $\checkmark$ \\
Reflexivity & $\checkmark$ & $\checkmark$ & $\checkmark$ & $\checkmark$ & $\checkmark$ \\
Cautious Monotony & $\times$ & $\checkmark$ & $\times$ & $\times$ & $\times$ \\
Cautious Transitivity & $\times$ & $\checkmark$ & $\checkmark$ & $\checkmark$ & $\times$ \\
\bottomrule
\end{tabular}
\caption{}
\end{table}

\begin{example}
According to Example 4, while it follows that $\Delta_1 \models allergies$ and $\Delta_1 \models sthrep\_ throat$, it also follows $\Delta_1, sthrep\_throat \not\models allergies$. Therefore, cautious monotony fails for the c- and s-consequence. To see that cautious transitivity fails, note that 
$\Delta_1 \models allergies$ and $\Delta_1, allergies \models sthrep\_ throat$. However, it also follows $\Delta_1 \not\models allergies \rightarrow sthrep\_ throat$. The remaining cases are proved by analogous reasoning.
\end{example}

\section{Metatheory}

This section introduces central results for the metatheory of the non-monotonic $\mathcal{AOL}$, namely, we prove the existence of a minimal model for any abductive problem \( \langle \Theta, \alpha \rangle \) within the framework of plausibility Kripke models, and show the conditions under which the selection method coincides with the preferential consequence relation. We begin by stating the following theorem:

\begin{theorem}[Minimal Model Existence]
For any abductive problem \( \langle \Theta, \alpha \rangle \) with a plausibility Kripke model \( \mathcal{M} = \langle \mathcal{W}, \mathcal{R}, \prec, v \rangle \), there exists a minimal model \( \mathcal{M}_{\text{min}} = \langle \mathcal{W}_{\text{min}}, \mathcal{R}_{\text{min}}, v_{\text{min}} \rangle \), where \( \mathcal{W}_{\text{min}} \) is the set of all \( \alpha \)-minimal states in \( \mathcal{W} \), and \( (\mathcal{M}_{\text{min}}, w ) \models \alpha \).
\end{theorem}

\subsection*{Proof of minimal model existence}

The proof consists in showing that given an arbitrary model $\mathcal{M}$ it is possible to construct a minimal model $\mathcal{M}_{min}$ that validates the same formulas. We start by introducing basic definitions.

\begin{mydef}
\label{minimalm}
Given a plausibility Kripke model $\mathcal{M} = \langle \mathcal{W}, \mathcal{R}, \prec, v \rangle$, define the minimal model
$\mathcal{M}_{min} = \langle \mathcal{W}_{min}, \mathcal{R}_{min}, v_{min} \rangle$ as follows: 
\begin{itemize}
\item \( \mathcal{W}_{\text{min}} = \{ w \in \mathcal{W} \mid \forall w' \in \mathcal{W}, w' \prec w \implies w' = w \} \),
\item $\mathcal{R}_{min} = \mathcal{R} \cap (\mathcal{W}_{min} \times \mathcal{W}_{min})$, 
\item $v_{min} = v(p) \cap \mathcal{W}_{min}$,
\end{itemize}
\end{mydef}

\begin{lemma}
The model $\mathcal{M}_{min} = \langle \mathcal{W}_{min}, \mathcal{R}_{min}, \prec, v_{min} \rangle$ is minimal with respect to the preferential relation $\prec$.
\end{lemma}

\begin{proof}
Straight from Definition \ref{minimalm} and construction of the model $\mathcal{M}_{min}$.
\end{proof}

\begin{lemma}
\label{val1}
 For any state $w \in \mathcal{W}_{min}$, $(\mathcal{M}, w) \models \varphi$ implies $(\mathcal{M}_{min}, w) \models \varphi$.
\end{lemma}

\begin{proof}
Let \( \mathcal{M} = \langle \mathcal{W}, \mathcal{R}, \prec, v \rangle \) be a plausibility Kripke model, and let \( \mathcal{M}_{\text{min}} = \langle \mathcal{W}_{\text{min}}, \mathcal{R}_{\text{min}}, v_{\text{min}} \rangle \) be the minimal model constructed as in Definition \ref{minimalm}. The proof is done by induction on the length of $\varphi$. For case \( \varphi = p \), where \( p \) is an atom. By construction, \( v_{\text{min}}(p) = v(p) \cap \mathcal{W}_{\text{min}} \). Thus, \( w \in v_{\text{min}}(p) \) if and only if \( w \in v(p) \). Therefore \( (\mathcal{M}_{\text{min}}, w) \models p \).
 For case \( \varphi = \neg \psi \). By induction hypothesis, \( (\mathcal{M}, w) \models \phi \) implies \( (\mathcal{M}_{\text{min}}, w) \models \phi \). Hence, \( (\mathcal{M}, w) \not\models \phi \) implies \( (\mathcal{M}_{\text{min}}, w) \not\models \phi \). Given the clause for negation, \( (\mathcal{M}_{\text{min}}, w) \models \neg \phi \).
For case \( \varphi = K \psi \) or \( O \psi \), recall that the accessibility relation \( \mathcal{R}_{\text{min}} \) is a restriction of \( \mathcal{R} \) to \( \mathcal{W}_{\text{min}} \). Hence it follows that \( (\mathcal{M}, w) \models \psi \) implies \( (\mathcal{M}_{\text{min}}, w) \models \psi \).

\end{proof}

\begin{theorem}[Minimal Model Existence]
For any abductive problem \( \langle \Theta, \phi \rangle \) with a plausibility Kripke model \( \mathcal{M} = \langle \mathcal{W}, \mathcal{R}, \prec, v \rangle \), there exists a minimal model \( \mathcal{M}_{\text{min}} = \langle \mathcal{W}_{\text{min}}, \mathcal{R}_{\text{min}}, v_{\text{min}} \rangle \), where \( \mathcal{W}_{\text{min}} \) is the set of all \( \phi \)-minimal states in \( \mathcal{W} \), and \( (\mathcal{M}_{\text{min}}, w ) \models \phi \).
\end{theorem}

\begin{proof}
Assume $\langle \Theta, \varphi \rangle$ is an abduction problem in a model $\mathcal{M}$. By Lemma \ref{val1}, we know there is minimal model 
$\mathcal{M}_{min}$ that validates $\langle \Theta, \varphi \rangle$. Now assume there is a solution $\alpha$ to the abduction problem $\langle \Theta, \varphi \rangle$ such that $(\mathcal{M}, w) \models \alpha$. Hence  $(\mathcal{M}_{min}, w) \models \alpha$. By definition of minimality, $\alpha$ is a minimal explanation.
\end{proof}

\subsection*{Selection of minimal sets and preferential consequence}

The following results show the conditions under which some selection method for minimal sets of explanations coincides with the preferential consequence relation relative to an abductive problem \( \langle \Theta, \alpha \rangle \).

\begin{mydef}
Given an abduction problem \( \langle \Theta, \alpha \rangle \), a set of explanations \( \Delta \subseteq \Pi(\Theta, \alpha) \) is \textbf{subset-minimal} if there is no proper subset \( \Delta' \subset \Delta \) such that \( \Theta \cup \Delta' \models \alpha \).
\end{mydef}

\begin{lemma}
Let $\Delta$ be a subset-minimal set of explanations for the abduction problem \( \langle \Theta, \alpha \rangle \). Therefore $\Theta \cup \Delta \models_{\prec} \alpha$. 
\end{lemma}

\begin{proof}
Let \( \Delta \) be a subset-minimal set of explanations and assume $(\mathcal{M}, w) \models \Theta, (\mathcal{M}, w) \models \Delta, \text{and } (\mathcal{M}, w) \models \alpha$ hold for a state \( w \in \mathcal{W} \). Since $\Delta$ is subset minimal, for any proper subset \( \Delta' \subset \Delta \), \( \Theta \cup \Delta' \not\models \alpha \). As a consequence, \( w \) is \( \alpha \)-minimal under \( \prec \),  for there is no \( w' \prec w \) such that \( (M, w') \models \Theta \) and \( (M, w') \models \alpha \). Finally, by definition of preferential consequence \( \Theta \cup \Delta \models_{\prec} \alpha \) follows.
\end{proof}

\begin{lemma}
Let $\Theta \cup \Delta \models_{\prec} \alpha$. Then $\Delta$ is subset-minimal.   
\end{lemma}

\begin{proof}
 Assume $\Theta \cup \Delta \models_{\preceq} \alpha$ holds and let \( \Delta \) be the set of explanations corresponding to a \( \alpha \)-minimal state in which the following are valid: $(\mathcal{M}, w) \models \Theta, (\mathcal{M}, w) \models \Delta, \text{and } (\mathcal{M}, w) \models \alpha$. Now, assume for a contradiction that $\Delta$ is not subset-minimal. Hence there is a proper subset \( \Delta' \subset \Delta \) such that \( \Theta \cup \Delta' \models \alpha \). Moreover, there is a state $w' \in \mathcal{W}$ such that $ w'\prec w$ and validates the following: $(\mathcal{M}, w') \models \Theta, (\mathcal{M}, w') \models \Delta, \text{and } (\mathcal{M}, w') \models \alpha$. However, by the connectedness of $\prec$ and the definition of minimality, a contradiction follows.
\end{proof}

\begin{theorem}
Given a set $\Delta$ of minimal explanations for the abduction problem \( \langle \Theta, \alpha \rangle \), $\Delta$ is subset-minimal iff \( \Theta \cup \Delta \models_{\preceq} \alpha \).    
\end{theorem}

\begin{proof}
By Lemmas 3 and 4.
\end{proof}

\begin{corollary}
For any abductive problem \( \langle \Theta, \alpha \rangle \), if the plausibility relation \( \prec \) on the set of states \( \mathcal{W} \) is transitive and connected, the preferential consequence relation \( \models_{\preceq} \) coincides with subset minimality in selecting minimal explanations. 
\end{corollary}

Analogous results can be proved for cardinality and the priorization selections of sets of minimal explanations. In fact, one can show that the cardinality and the subset-minimality selections always coincide in selecting minimal explanations for a finite knowledge base.










\section{Related work and future direction}

The logics of only-knowing were introduced by H. Levesque in \cite{levesque1990all} and later explored by distinct authors (\cite{halpern2001multi}, \cite{chen1994logic} and \cite{belle2010multi}) as a way of modelling different aspects of knowledge dynamics. In particular, the problem of developing knowledge-based accounts of abduction via $\mathcal{OL}$ was first introduced by Levesque in \cite{levesque1989knowledge}.  In this paper, we explored an alternative account by treating abduction as an epistemic modality defined via the combination of basic epistemic concepts and within the bounds of the background knowledge of the agent. Furthermore, a non-monotonic consequence relation was built over the monotonic logic $\mathcal{AOL}$ by augmenting the semantic framework via plausibility Kripke models.

Different modal approaches to abduction are found in authors such as \cite{mayer1995propositional}, \cite{levesque1990all} and \cite{gauderis2013modelling}. While most current approaches to abduction seek to develop its dynamic reasoning through a process of discovery modeled by proof-theoretical means (see \cite{mayer1995propositional} or \cite{soler2009abduction}), our work introduced a semantical approach via the introduction of a new modality in a modal (and preferential) structure. The development of adequate proof-theoretic tools for the resulting logic is a topic for future exploration. Furthermore, while a first connection between $\mathcal{OL}$ and non-monotonic formalism was introduced by Levesque (\cite{levesque1989knowledge}) through an interpretation of autoepistemic logic, we introduced an alternative interpretation via $\mathcal{AOL}$ within a preferential modal semantics. Our framework represents a novel approach compared to other non-monotonic modal frameworks for abductive reasoning (such as \cite{meheus2002ampliative}).

The results demonstrated in Section 4 and 5 show that the non-monotonic logic built over the $\mathcal{AOL}$ is capable of expressing different methods for selection of minimal explanations. The fact that $\mathcal{AOL}$ and $\mathcal{AOL}^\prec$ enjoy central properties for logic-based characterization of abduction suggests that these systems can be provided with sequent rules for abduction (in the style proposed by \cite{aliseda2006abductive}). Furthermore, other expansions of $\mathcal{OL}$ to multi-modal framework are introduced by Levesque in \cite{levesque1989knowledge}. The generalization of our framework to multi-agent scenarios or to other multi-modal frameworks for only-knowing remains a path for future exploration. 

\section{Acknowledgements}

The authors thank the careful reading of the reviewers.


\bibliographystyle{eptcs}
\bibliography{generic}
\end{document}